\pdfoutput=1
\documentclass[11pt]{article}

\usepackage[margin=1in]{geometry}
\usepackage[T1]{fontenc}
\usepackage{lmodern}    
\usepackage{microtype}

\usepackage{amsmath,amssymb,amsthm,mathtools}

\usepackage{float}
\newfloat{algorithm}{t}{loa}
\floatname{algorithm}{Algorithm}

\usepackage{booktabs}
\usepackage{graphicx}
\usepackage{xcolor}
\usepackage{subcaption}
\usepackage{tikz}
\usetikzlibrary{arrows.meta,positioning,fit,backgrounds,shapes.geometric,calc}
\graphicspath{{figures/}}

\usepackage[numbers,sort&compress]{natbib}
\usepackage[colorlinks=true,linkcolor=black,citecolor=black,urlcolor=blue]{hyperref}

\newtheorem{proposition}{Proposition}

\definecolor{acc}{HTML}{2C6E8F}
\definecolor{warm}{HTML}{E07B39}
\definecolor{good}{HTML}{4C9A6E}

\newcommand{\cfg}[1]{\texttt{#1}}
\newcommand{\x}{\times}

\title{\bfseries Lossless but Not Free:\\[2pt]
An Empirical Anatomy of Speculative Decoding\\ on Consumer Hardware}

\author{
  Param Chordiya\\
  University of California, San Diego\\
  \texttt{pchordiya@ucsd.edu}
}
\date{\today}

\begin{document}
\maketitle

\begin{abstract}
Single-stream autoregressive decoding of large language models is bound by memory
bandwidth: each generated token requires one full forward pass through the target model,
and successive passes cannot be parallelized. Speculative decoding restructures this
computation --- a small draft model proposes $K$ tokens autoregressively, and the target
model scores all of them in one batched pass --- while a rejection-sampling rule provably
preserves the target model's output distribution. We present a from-scratch,
device-agnostic (CUDA/MPS/CPU) implementation and an empirical study across five
draft/target backend configurations on a consumer Apple-silicon laptop. Distribution
equivalence is verified at three levels, culminating in a two-sample test over roughly
$9{,}200$ real-model tokens per method ($\chi^2 = 162.5$, $\mathrm{dof} = 200$, $p = 0.976$)
and exact greedy-sequence agreement. The best configuration reaches a measured
\textbf{$1.61\x$} wall-clock speedup at $K{=}6$, on an acceptance profile declining from
$69.7\%$ at $K{=}1$ to $37.8\%$ at the optimum, while three of five configurations
\emph{decelerate} --- because the draft fails to out-speed a small target, or because the
quantized Metal backend executes ``parallel'' verification serially, an effect we isolate
and quantify. The failures are as instructive as the successes: speculative decoding pays
off only when verification is genuinely batch-parallel and the draft/target latency gap is
real.
\end{abstract}

\noindent\textbf{Keywords:} speculative decoding, speculative sampling, LLM inference,
rejection sampling, KV cache, quantization, Apple silicon, llama.cpp, Metal, MPS

\section{Introduction}

Generating text from a decoder-only transformer is a serial process. Producing token $t+1$
requires a forward pass conditioned on tokens $1..t$; producing token $t+2$ requires the
output of that pass. No amount of parallel hardware removes this dependency chain --- at
batch size one, each decoding step reads essentially all model weights to produce a single
token, so the process is bound by memory bandwidth rather than arithmetic
throughput~\citep{kwon2023efficient,yu2022orca}. A modern accelerator that can multiply
matrices at hundreds of TFLOP/s spends a single-stream decode step mostly waiting on weight
reads.

Speculative decoding~\citep{leviathan2023fast,chen2023accelerating} attacks the dependency
chain itself. The observation is that \emph{verifying} a proposed continuation is
parallelizable even though \emph{generating} one is not: if some cheap process guesses the
next $K$ tokens, the target model can score all $K$ positions in one forward pass, because
every position's inputs are known in advance. A small draft model supplies the guesses; a
carefully constructed accept/reject rule keeps only the prefix of guesses the target model
``agrees with,'' and --- critically --- resamples the first rejected position from an
adjusted distribution such that the overall output distribution is \emph{mathematically
identical} to sampling the target model alone. Acceleration therefore never trades away
output quality; the only question is whether the bookkeeping costs less than the target
passes it saves.

That question is usually answered on server-class GPUs. This paper answers it on consumer
hardware --- an Apple-silicon laptop with 18\,GB of unified memory --- where quantized
inference runtimes, Metal compute kernels, and PyTorch's MPS backend interact in ways that
can silently invalidate speculative decoding's premises. Our contributions:

\begin{itemize}
  \item \textbf{A from-scratch, device-agnostic implementation.} A five-block engine (model
  abstraction, draft--verify loop, rejection sampling, benchmark harness, CLI) that runs
  unmodified on CUDA, MPS, and CPU, mixes backends freely (Hugging Face Transformers fp16,
  llama.cpp GGUF 4/8-bit), and performs all sampling arithmetic in float32 on the CPU with
  an explicit generator for cross-device reproducibility.
  \item \textbf{A verified distribution-equivalence result.} Correctness is gated, not
  assumed, at three levels: a $50{,}000$-trial statistical test of the rejection-sampling
  rule against analytic ground truth; an end-to-end decoder test on deterministic mock
  models including bit-exact greedy agreement; and a real-model gate comparing about
  $9{,}200$ speculative tokens against as many target-only tokens (a $\chi^2$ two-sample
  test, $p = 0.976$).
  \item \textbf{An empirical $K$-sweep across five backend configurations,} measuring
  throughput, acceptance by draft position, time-to-first-token, latency decomposition, and
  peak memory. The sweep surfaces two failure modes that we isolate with targeted
  microbenchmarks: a draft that is slower per step than a small quantized target, and a
  quantized Metal backend whose verification cost grows \emph{linearly} in batch size for
  batches of $2$--$7$ because it executes matrix--vector kernels per position.
\end{itemize}

\section{Background and Related Work}

\paragraph{Autoregressive decoding is latency-bound.}
A decoder-only transformer with cached keys and values performs, per generated token, one
matrix--vector-shaped pass over its weights. At batch size one the arithmetic intensity of
this pass is far below the compute/bandwidth break-even of modern accelerators, so step
latency approaches (weights read)$/$(memory bandwidth) --- the \emph{memory-bandwidth-bound}
regime. Larger batches amortize the weight reads across many streams and move inference
toward the \emph{compute-bound} regime; this is exactly the lever that serving-side
optimizations pull. Continuous batching schedules many concurrent requests at iteration
granularity~\citep{yu2022orca}, and PagedAttention partitions KV-cache memory into paged
blocks so those batches fit~\citep{kwon2023efficient}. Both raise \emph{throughput} across
requests; neither shortens the latency of a single stream, which is the setting this paper
targets.

\paragraph{Orthogonal accelerations.}
Weight quantization --- GPTQ's second-order one-shot approach~\citep{frantar2023gptq}, AWQ's
activation-aware scaling~\citep{lin2024awq}, NF4 4-bit data types~\citep{dettmers2023qlora},
and llama.cpp's k-quant formats~\citep{gerganov2023llamacpp} --- shrinks the weights that
each step must read and is fully composable with speculative decoding: our target models are
4-bit quantized. Knowledge distillation~\citep{hinton2015distilling} transfers a large
model's behavior into a smaller one; it changes the model rather than the decoding
procedure, and (unlike speculative decoding) offers no guarantee of matching the large
model's distribution.

\paragraph{Speculative decoding.}
Blockwise parallel decoding~\citep{stern2018blockwise} first exploited the
verify-in-parallel observation, using auxiliary prediction heads and exact-match acceptance
for greedy decoding. \citet{leviathan2023fast} and \citet{chen2023accelerating}
independently generalized this to \emph{sampling} with a modified rejection rule that
preserves the target distribution exactly (Section~\ref{sec:method}), using a separate small
draft model. Follow-on work varies the draft source: Medusa~\citep{cai2024medusa} attaches
several additional decoding heads to the target model itself and verifies a \emph{tree} of
candidates under a purpose-built attention mask; self-speculative approaches such as Draft
\& Verify~\citep{zhang2023draft} skip layers of the target to draft, requiring no second
model. We implement the two-model formulation of \citet{leviathan2023fast} and
\citet{chen2023accelerating} because it cleanly separates the mathematical guarantee
(rejection sampling) from the systems question (when do the economics work), which is the
object of our study.

\section{Method}
\label{sec:method}

\subsection{Setting and notation}

Let $p(\cdot \mid s)$ denote the target model's next-token distribution given context $s$,
and $q(\cdot \mid s)$ the draft model's. Both are computed as temperature-scaled softmaxes
over a shared vocabulary $V$; the same temperature $T$ is applied to both.
Table~\ref{tab:notation} collects the symbols used throughout.

\begin{table}[t]
\centering
\caption{Notation.}
\label{tab:notation}
\small
\begin{tabular}{ll}
\toprule
Symbol & Meaning \\
\midrule
$p(\cdot\mid s),\ q(\cdot\mid s)$ & target / draft next-token distribution given context $s$ \\
$V,\ T$ & shared vocabulary; sampling temperature (applied to both models) \\
$K$ & speculation length (draft tokens proposed per round) \\
$x_i,\ q_i$ & the $i$-th drafted token and the draft distribution it was sampled from \\
$p_i$ & target distribution at position $i$, i.e.\ $p(\cdot\mid s+x_{1:i-1})$ \\
$p_{K+1}$ & bonus distribution $p(\cdot\mid s+x_{1:K})$, free from the verify pass \\
$r_i$ & acceptance draw, $r_i\sim\mathrm{Uniform}(0,1)$ \\
$a$ & index of the last accepted draft (first rejection at $a{+}1$) \\
$p'_{a+1}$ & adjusted resampling distribution at the first rejected position \\
$\alpha$ & (pooled) acceptance rate \\
\bottomrule
\end{tabular}
\end{table}

One round of speculative decoding proceeds:
\begin{enumerate}
  \item \textbf{Draft.} Sample $K$ tokens autoregressively from the draft:
  $x_i \sim q(\cdot \mid s + x_{1:i-1})$ for $i = 1..K$, retaining each full distribution
  $q_i$.
  \item \textbf{Verify.} Run the target once over the $K$ new positions (with its KV cache
  covering $s$), obtaining $K{+}1$ distributions: $p_i = p(\cdot \mid s + x_{1:i-1})$ for
  $i = 1..K$, plus the \emph{bonus} distribution $p_{K+1} = p(\cdot \mid s + x_{1:K})$.
  \item \textbf{Accept/reject.} Walk $i = 1..K$: draw $r_i \sim \mathrm{Uniform}(0,1)$ and
  accept $x_i$ iff
  \begin{equation}
    r_i < \min\!\left(1,\ \frac{p_i(x_i)}{q_i(x_i)}\right).
    \label{eq:accept}
  \end{equation}
  At the first rejection, position $a{+}1$, discard $x_{a+1}..x_K$ and emit one replacement
  token drawn from the \emph{adjusted} distribution
  \begin{equation}
    p'_{a+1}(x) = \frac{\max\!\big(0,\; p_{a+1}(x) - q_{a+1}(x)\big)}
                       {\sum_{y \in V} \max\!\big(0,\; p_{a+1}(y) - q_{a+1}(y)\big)}.
    \label{eq:adjust}
  \end{equation}
  \item \textbf{Bonus.} If all $K$ drafts are accepted, emit one extra token sampled from
  $p_{K+1}$ --- free, because that distribution was already computed by the verify pass.
\end{enumerate}
The round emits between $1$ and $K{+}1$ tokens for exactly one target pass and $K$ draft
steps; the loop repeats until an end-of-sequence token or a length budget
(Algorithm~\ref{alg:round}).

\begin{algorithm}[t]
\small
\hrule height 0.9pt\vspace{4pt}
\noindent\textbf{Input:}\ context $s$; draft $q$, target $p$; speculation length $K$; temperature $T$\\
\textbf{Output:}\ $1$ to $K{+}1$ new tokens, distributed as $p(\cdot\mid s)$\\[2pt]
\hrule height 0.4pt\vspace{4pt}
\noindent\textbf{for} $i = 1$ \textbf{to} $K$:\quad $x_i \sim q(\cdot \mid s + x_{1:i-1})$, retaining $q_i$ \hfill $\triangleright$ draft autoregressively\\
$(p_1,\dots,p_{K+1}) \leftarrow \textsc{TargetForward}(s + x_{1:K})$ \hfill $\triangleright$ one target pass, $K{+}1$ rows\\
\textbf{for} $i = 1$ \textbf{to} $K$:\\
\hspace*{1.6em}draw $r_i \sim \mathrm{Uniform}(0,1)$\\
\hspace*{1.6em}\textbf{if} $r_i < \min\!\big(1,\ p_i(x_i)/q_i(x_i)\big)$:\ \ commit $x_i$ \hfill $\triangleright$ accept\\
\hspace*{1.6em}\textbf{else}:\ \ emit $y \sim p'_i$ with $p'_i(x)\propto \max\!\big(0, p_i(x)-q_i(x)\big)$; \textbf{return} committed $\,+\, y$ \hfill $\triangleright$ reject; discard $x_{i:K}$\\
emit bonus $y \sim p_{K+1}$; \textbf{return} committed $\,+\, y$ \hfill $\triangleright$ all $K$ accepted\\[2pt]
\hrule height 0.9pt
\caption{One round of speculative decoding. Caches are truncated to the accepted length after
the walk; sampling arithmetic is CPU float32.}
\label{alg:round}
\end{algorithm}

\subsection{Why the output distribution equals the target's}
\label{sec:proof}

\begin{proposition}[\citealp{leviathan2023fast,chen2023accelerating}]
For any context, the token emitted by one round of Algorithm~\ref{alg:round} is distributed
exactly as $p(\cdot\mid s)$.
\end{proposition}

\begin{proof}[Proof sketch (single position)]
Fix a context and abbreviate $p, q$ for the target and draft distributions there. The
emitted token $Y$ is either an accepted draft or an adjusted resample. For any token $y$,
\begin{equation}
  \Pr[Y = y,\ \text{accept}] = q(y)\cdot\min\!\left(1, \frac{p(y)}{q(y)}\right)
  = \min\big(p(y),\, q(y)\big).
\end{equation}
The total rejection probability is
$1 - \sum_z \min(p(z), q(z)) = \sum_z \max(0, p(z) - q(z))$, which is exactly the normalizer
of $p'$ in Eq.~\eqref{eq:adjust}, so
\begin{equation}
  \Pr[Y = y,\ \text{reject}]
  = \Big[\textstyle\sum_z \max(0, p(z)-q(z))\Big] \cdot p'(y)
  = \max\big(0,\, p(y) - q(y)\big).
\end{equation}
Summing the two cases gives
$\Pr[Y = y] = \min(p,q)(y) + \max(0, p-q)(y) = p(y)$.
Induction over positions extends this to sequences: conditioned on any committed prefix, the
next emitted token is distributed as $p(\cdot \mid \text{prefix})$ --- whether it arrives as
an accepted draft, a resample, or a bonus draw (the bonus is a direct sample from
$p_{K+1}$). The marginal process is therefore indistinguishable from autoregressive sampling
of the target alone.
\end{proof}

\paragraph{Greedy special case.}
As $T \to 0$ both distributions collapse to one-hot argmax vectors, and the general rule
degenerates to: accept iff $\arg\max p_i = x_i$, else emit $\arg\max p_i$. Our implementation
treats $T{=}0$ explicitly (no RNG is consumed), and the general-rule-on-one-hots equivalence
is unit-tested.

\paragraph{Expected yield.}
Under the idealized assumption of a position-independent acceptance rate $\alpha$, a round
emits $(1-\alpha^{K+1})/(1-\alpha)$ tokens in expectation~\citep{leviathan2023fast}. Measured
acceptance in fact \emph{decays} with draft position (Section~\ref{sec:accept}), so this
expression is an upper bound that becomes loose as $K$ grows --- one of the two reasons large
$K$ underperforms in practice.\footnote{Eq.~\eqref{eq:adjust} is well defined whenever a
rejection can occur; the implementation additionally guards the measure-zero case in which
the residual mass is exactly zero (only reachable when $p=q$ at that position) by falling
back to $p$, which is distributionally harmless there.}

\section{Implementation}
\label{sec:impl}

\begin{figure}[t]
\centering
\resizebox{\linewidth}{!}{%
\begin{tikzpicture}[
  >={Latex[length=2mm]}, font=\small, node distance=7mm and 11mm,
  model/.style={rounded corners=3pt, draw=warm!85, fill=warm!8, align=center,
                inner sep=4pt, minimum height=12mm, text width=25mm},
  block/.style={rounded corners=3pt, draw=acc!80, fill=acc!7, align=center,
                inner sep=4pt, minimum height=12mm, text width=27mm},
  term/.style={rounded corners=7pt, draw=black!55, fill=black!5, align=center,
               inner sep=4pt, text width=17mm, minimum height=10mm},
  flow/.style={->, thick, draw=black!60},
  loop/.style={->, thick, dashed, draw=black!55},
]
\node[term] (ctx) {context $s$\\ (KV caches)};
\node[model, right=of ctx] (draft) {\textbf{Draft} $q$\\[1pt] $K$ autoregr.\ steps\\ keep $q_{1:K}$};
\node[model, right=of draft] (verify) {\textbf{Verify} $p$\\[1pt] one target pass\\ rows $p_{1:K+1}$};
\node[block, right=of verify] (accept) {\textbf{Accept walk}\\[1pt]
  $r_i<\min\!\big(1,\tfrac{p_i(x_i)}{q_i(x_i)}\big)$};
\node[block, above right=3mm and 12mm of accept] (bonus)
  {\textbf{Bonus}\\ $y\sim p_{K+1}$};
\node[block, below right=3mm and 12mm of accept] (resample)
  {\textbf{Resample}\\ $y\sim p'_{a+1}$\\ $\propto\max(0,p{-}q)$};
\node[term, right=30mm of accept] (commit) {commit\\ $1$--$(K{+}1)$\\ tokens};

\draw[flow] (ctx) -- (draft);
\draw[flow] (draft) -- node[above,font=\scriptsize]{$x_{1:K}$} (verify);
\draw[flow] (verify) -- (accept);
\draw[->, thick, draw=good!60!black] (accept) to[out=25,in=180]
  node[above left=0mm and 1mm,font=\scriptsize,text=good!55!black]{all $K$ acc.} (bonus);
\draw[->, thick, draw=red!65!black] (accept) to[out=-25,in=180]
  node[below left=0mm and 1mm,font=\scriptsize,text=red!60!black]{reject @ $a{+}1$} (resample);
\draw[flow] (bonus) to[out=0,in=90] (commit);
\draw[flow] (resample) to[out=0,in=-90] (commit);
\draw[loop] (commit) to[out=-90,in=-90]
  node[below,font=\scriptsize,pos=0.5]{append committed tokens; truncate KV caches; repeat}
  (ctx);
\node[font=\scriptsize\itshape, text=warm!60!black, below=1mm of draft, align=center]
  {HF fp16 (MPS) or\\ GGUF (llama.cpp/Metal)};
\end{tikzpicture}}
\caption{Core draft--verify--accept loop and the two exit branches. The draft proposes
$K$ tokens autoregressively (keeping each distribution $q_i$); the target scores all $K$
positions in one pass, yielding $p_{1:K+1}$; the accept walk commits the longest agreeing
prefix. If every draft is accepted, a free \emph{bonus} token is drawn from $p_{K+1}$;
otherwise the first rejected position is \emph{resampled} from the adjusted distribution and
the remaining drafts are discarded. Both model roles sit behind one device-agnostic
abstraction (Block~1) that mixes Hugging Face fp16 and llama.cpp GGUF backends. Both models
run on either backend; sampling arithmetic is always CPU float32.}
\label{fig:arch}
\end{figure}
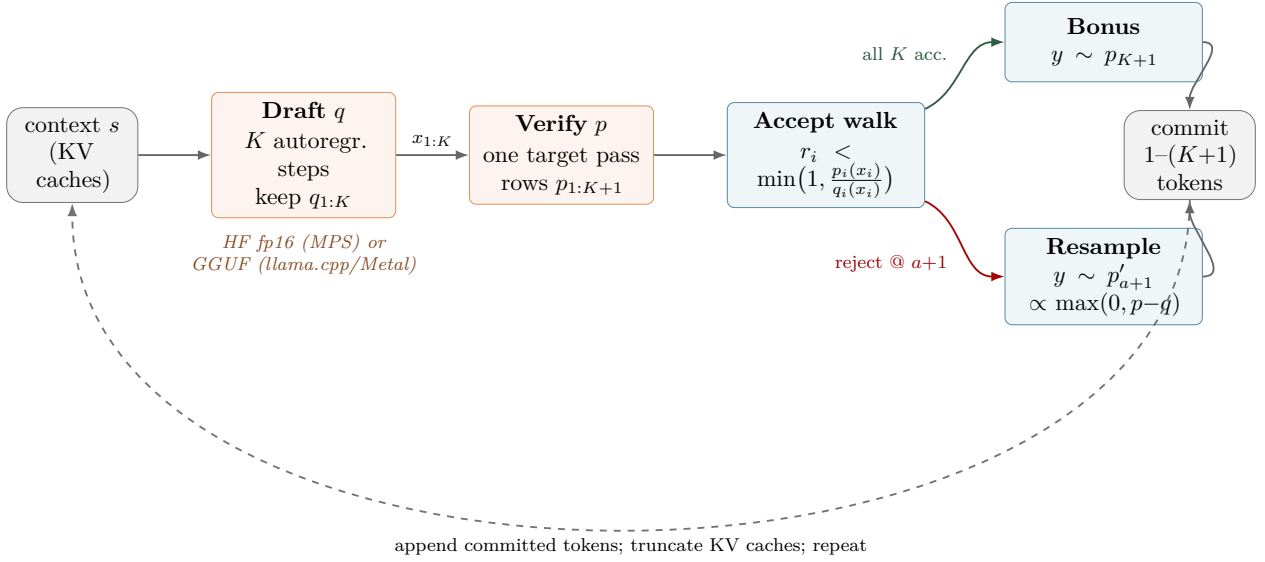

The engine is organized as five blocks coded against a frozen interface file that fixes
every shared signature, the KV-cache choreography, and the benchmark JSON schema before any
block was written --- which allowed the blocks to be developed in parallel without
integration drift. Figure~\ref{fig:arch} shows the core loop.

\paragraph{Model abstraction.}
A \cfg{SpecModel} interface exposes a stateless \cfg{forward} and a stateful incremental
API: \cfg{extend(tokens)}$\to$logits appends tokens to the model's KV cache and returns
full-vocabulary logits for the new positions; \cfg{truncate\_cache(n)} rolls the cache back
(a no-op when already shorter --- the contract that makes post-rejection cleanup uniform).
Two implementations exist. \cfg{HFModel} wraps any Hugging Face causal LM with a
\cfg{DynamicCache}, cropping it in place on rollback. \cfg{QuantizedModel} wraps llama.cpp:
the KV rollback sets \cfg{n\_tokens} and explicitly removes stale cache entries via the
current \cfg{kv\_cache\_seq\_rm} API, and --- essential for correctness --- the runtime is
configured with \cfg{logits\_all=True} so \emph{full-vocabulary} logits are available at
every drafted position, not just the final one; rejection sampling is impossible with
top-$k$-truncated outputs. Both wrappers pass a KV-consistency test battery: after any
extend/truncate sequence, logits must match a fresh evaluation of the same tokens.

\paragraph{One tokenizer.}
The draft's Hugging Face tokenizer is the single source of truth; llama.cpp receives token
ids directly and never re-tokenizes. Pair construction validates tokenizer identity across
backends (vocabulary comparison for HF--HF pairs; id$\to$text round-trips for HF--GGUF
pairs) and computes the shared sampling width as the minimum of the two logit widths, since
model families pad their output layers differently. This validation exists because
plausible-looking pairs fail it: TinyLlama~\citep{zhang2024tinyllama} uses a $32{,}000$-token
Llama-2 vocabulary and cannot legally pair with Llama-3.1's $128{,}256$-token one, which is
why our CUDA headline script pairs Llama-3.2-1B with
Llama-3.1-8B~\citep{grattafiori2024llama3} instead.

\paragraph{Numerics and reproducibility.}
All sampling mathematics --- temperature scaling, the accept test, the adjusted-distribution
construction, multinomial draws --- executes on CPU in float32 through an explicit
\cfg{torch.Generator}, regardless of where the models run. Model wrappers return CPU float32
logits at their boundary. This costs microseconds per round (Section~\ref{sec:time}) and
buys bit-reproducible runs across CUDA, MPS, and CPU, immunity to fp16 underflow in
probability space, and a well-defined division in $p/q$ (with the $q = 0$ edge cases handled
explicitly).

\paragraph{Cache choreography.}
Between rounds both models' caches hold a prefix of the committed sequence; each round
extends the draft cache with the uncached suffix, extends the target cache once with (last
committed token $+$ all $K$ drafts), takes the last $K{+}1$ logit rows --- the alignment
where row $i$ predicts draft token $i{+}1$ --- and truncates both caches back to the
accepted length after the accept walk. Getting the ``last $K{+}1$ rows'' alignment wrong is
the classic off-by-one in speculative decoders; it is pinned by mock-model tests in which
any misalignment breaks a cache-prefix invariant assertion.

\paragraph{Device handling.}
A single \cfg{get\_device()} helper selects CUDA/MPS/CPU; timers call a device-appropriate
synchronize before reading the clock. The same benchmark scripts run unmodified on a CUDA
machine; a self-contained Colab script reproduces the benchmark JSON schema for the 8B-class
pairing (Llama-3.2-1B fp16 draft $+$ Llama-3.1-8B 4-bit target via bitsandbytes
NF4~\citep{dettmers2023qlora}).

\section{Experimental Setup}
\label{sec:setup}

\paragraph{Hardware and software.}
All measurements in this paper were produced on one consumer machine: an Apple-silicon
(M3-class, arm64) laptop with 18\,GB of unified memory, macOS~26.5; Python~3.12.2,
PyTorch~2.12.1 (MPS backend), Transformers~5.13.0, llama-cpp-python~0.3.33 (Metal). No CUDA
device was available; the 8B-class configuration is provided as a runnable Colab-T4 script
that emits the identical results schema, but its numbers are deliberately absent from this
paper because they were not measured here.

\paragraph{Models.}
All local pairs draw from the Qwen2.5 instruct family~\citep{qwen2024qwen25}, which shares
one tokenizer across sizes: Qwen2.5-0.5B (draft; HF fp16 or GGUF
q8\_0)~\citep{qwen2024qwen25_05b}, and targets Qwen2.5-1.5B (GGUF q4\_k\_m or HF
fp16)~\citep{qwen2024qwen25_15b}, Qwen2.5-3B (HF fp16)~\citep{qwen2024qwen25_3b}, Qwen2.5-7B
(GGUF q4\_k\_m)~\citep{qwen2024qwen25_7b}. Five configurations cover the backend
cross-product (Table~\ref{tab:configs}).

\begin{table}[t]
\centering
\caption{The five benchmark configurations (draft $\to$ target, with backend).}
\label{tab:configs}
\small
\begin{tabular}{lll}
\toprule
Config & Draft (backend) & Target (backend) \\
\midrule
\cfg{primary}    & 0.5B fp16 (Transformers, MPS)   & 1.5B q4\_k\_m (llama.cpp, Metal) \\
\cfg{all\_gguf}  & 0.5B q8\_0 (llama.cpp, Metal)   & 1.5B q4\_k\_m (llama.cpp, Metal) \\
\cfg{stretch\_7b}& 0.5B q8\_0 (llama.cpp, Metal)   & 7B q4\_k\_m (llama.cpp, Metal) \\
\cfg{mixed}      & 0.5B q8\_0 (llama.cpp, Metal)   & 1.5B fp16 (Transformers, MPS) \\
\cfg{mixed\_3b}  & 0.5B q8\_0 (llama.cpp, Metal)   & 3B fp16 (Transformers, MPS) \\
\bottomrule
\end{tabular}
\end{table}

\paragraph{Prompts.}
A fixed, vendored set of $100$ instruction-style prompts ($48$--$788$ characters;
explanation, summarization-with-passage, reasoning, coding, how-to, comparison, and creative
categories) is versioned in the repository for offline reproducibility; no benchmark-time
dataset download occurs. Prompts are chat-templated identically for both decoding paths.

\paragraph{Protocol.}
Main benchmarks: the first $30$ prompts, $2$ used as untimed warm-up ($28$ recorded), $128$
new tokens per prompt, $T = 0.8$, $K = 4$, per-prompt seed $= \text{prompt index}$.
$K$-sweeps: $10$ prompts ($2$ warm-up), $96$ tokens, $K \in \{1, 2, 4, 6, 8, 10, 12\}$.
Equivalence gate: $100$ prompts $\times$ $100$ tokens per method at $T = 1.0$, seeds
$1234{+}i$, plus a greedy leg of $5$ prompts $\times$ up to $200$ tokens at $T = 0$. Every
run writes a JSON artifact; every number below cites its artifact.

\paragraph{Metrics.}
\emph{Throughput}: generated tokens $/$ wall-clock per prompt, reported mean $\pm$ std across
recorded prompts, with a device synchronize before each clock read. \emph{TTFT}: start of
generation to the first committed token (for speculative decoding this includes the first
full draft$+$verify round). \emph{Acceptance rate}: accepted draft tokens $/$ drafted tokens,
pooled over the run; \emph{positional acceptance} is the same statistic restricted to draft
position $i$. \emph{Latency split}: draft $/$ verify $/$ sampling $/$ other wall-time per
generated token, from per-round timers. \emph{Peak memory}: process-wide high-water RSS
(\cfg{ru\_maxrss}); as a high-water mark it is non-decreasing across runs in one process, so
per-config values are upper bounds. VRAM is absent because no CUDA device was present.

\section{Results}

\subsection{Distribution equivalence (the correctness gate)}
\label{sec:equiv}

Three levels of verification, from algebra to real models.

\paragraph{Level 1 --- the sampling rule in isolation.}
For four synthetic $(p, q)$ families over a $30$-token vocabulary (uniform-Dirichlet,
heavy-tailed, near-identical, mismatched concentration), $50{,}000$ trials of draw-from-$q$
$\to$ accept/reject $\to$ adjusted-resample were compared against the analytic target $p$.
Goodness-of-fit $p$-values were $0.652$, $0.260$, $0.225$, and $0.978$, with
$\mathrm{KL}(\hat{p} \,\|\, p) \le 3.5\times10^{-4}$ in all cases --- the emitted marginal is
statistically indistinguishable from $p$.

\paragraph{Level 2 --- the full decoder on deterministic mock models.}
With draft and target implemented as fixed stochastic matrices (vocabulary $48$), pooling
$1{,}000$ seeded generations of $25$ tokens per method gave a two-sample $\chi^2 = 67.0$
($\mathrm{dof} = 47$, $p = 0.029$, above the $0.01$ gate threshold and consistent with the
null) over $20{,}950$ tokens per method, and greedy decoding produced \emph{bit-identical}
sequences to the baseline for every $K \in \{1, 2, 4, 8\}$.

\paragraph{Level 3 --- real models.}
Using the \cfg{primary} pair, $100$ prompts $\times$ $100$ tokens at $T = 1.0$ produced
$9{,}209$ speculative and $9{,}214$ baseline tokens under matched seeds. A two-sample
$\chi^2$ test over $201$ frequency bins (binned so every expected count $\ge 5$) gives
\begin{equation}
  \chi^2 = 162.5,\qquad \mathrm{dof} = 200,\qquad p = 0.976,
\end{equation}
--- no detectable difference between the speculative output distribution and target-only
sampling (Figure~\ref{fig:overlay}). The greedy leg produced exactly identical token
sequences on $5/5$ prompts ($200, 122, 147, 200, 200$ tokens; the shorter two stopped at EOS
in both paths).

\begin{figure}[t]
\centering
\includegraphics[width=0.72\linewidth]{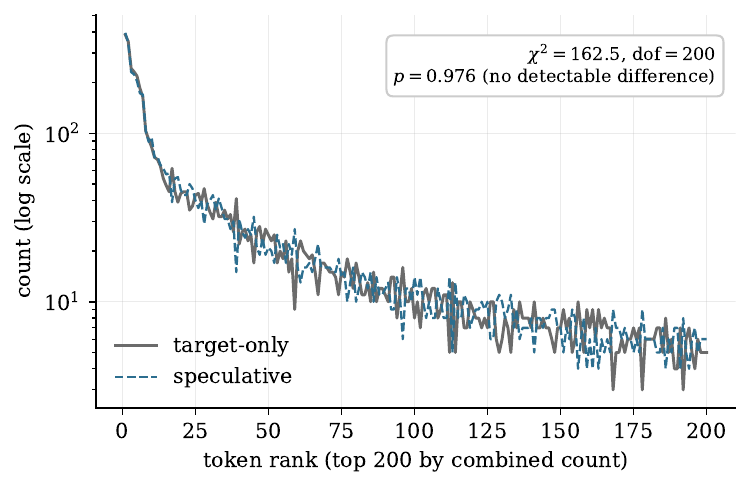}
\caption{Token-frequency overlay for the top-$200$ tokens by combined count, speculative vs.\
target-only, on a log scale, annotated with the two-sample test statistics
(\cfg{primary} pair, $100$ prompts $\times$ $100$ tokens, $T{=}1.0$). The two curves are
visually and statistically indistinguishable ($p = 0.976$).}
\label{fig:overlay}
\end{figure}

Two honest caveats. First, the reported smoothed symmetric KL ($0.471$) is dominated by
add-$0.5$ smoothing over the sparse union support (about $9{,}200$ samples spread thinly
over a large vocabulary); the calibrated instrument is the binned $\chi^2$. Second, the
greedy leg runs both paths on an fp32-CPU backend. The reason is measured, not asserted: on
the quantized Metal backend, evaluating the \emph{same context} with the final tokens in a
batch of five versus one at a time shifts logits by up to roughly $0.1$ --- enough to flip
the argmax at near-ties (observed top-1/top-2 probability gaps of
$1.8\times10^{-3}$--$5.8\times10^{-3}$), with no speculative machinery involved.
Batch-invariance of logits is a \emph{backend} property that bit-exact greedy comparison
presupposes; fp32 CPU exhibits it, the quantized Metal path does not.

\subsection{Throughput, speedup, and the \texorpdfstring{$K$}{K}-sweep}
\label{sec:sweep}

\begin{table}[t]
\centering
\caption{Main results ($K = 4$, $T = 0.8$, $128$ new tokens, mean $\pm$ std over $28$
prompts). Peak RSS is a process-wide high-water mark (an upper bound). Values from each run's
JSON aggregate; speedups from \texttt{results/summary.json}.}
\label{tab:main}
\small
\setlength{\tabcolsep}{5pt}
\begin{tabular}{lcccccccc}
\toprule
Config & Base tok/s & Spec tok/s & Speedup & Accept. & TTFT$_\text{b}$(s) & TTFT$_\text{s}$(s) & Peak RSS (GB) \\
\midrule
\cfg{mixed\_3b}  & $16.00{\pm}0.71$ & $\mathbf{25.07{\pm}2.59}$ & $\mathbf{1.57\x}$ & $43.3\%$ & $0.237$ & $0.389$ & $5.23$ \\
\cfg{mixed}      & $26.38{\pm}2.70$ & $36.91{\pm}6.53$ & $1.40\x$ & $51.0\%$ & $0.187$ & $0.268$ & $3.78$ \\
\cfg{all\_gguf}  & $73.99{\pm}6.34$ & $37.00{\pm}4.74$ & $0.50\x$ & $52.4\%$ & $0.125$ & $0.248$ & $3.22$ \\
\cfg{stretch\_7b}& $23.15{\pm}2.16$ & $11.95{\pm}1.35$ & $0.52\x$ & $40.2\%$ & $0.479$ & $0.736$ & $5.49$ \\
\cfg{primary}    & $73.46{\pm}6.25$ & $24.44{\pm}3.10$ & $0.33\x$ & $53.4\%$ & $0.122$ & $0.307$ & $2.93$ \\
\bottomrule
\end{tabular}
\end{table}

\begin{figure}[t]
\centering
\begin{subfigure}{0.49\linewidth}
  \centering\includegraphics[width=\linewidth]{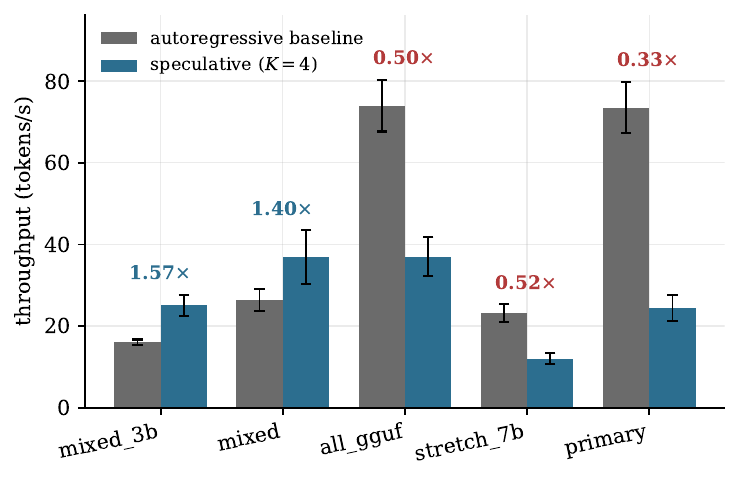}
  \caption{Baseline vs.\ speculative throughput per configuration, $K{=}4$.}
  \label{fig:speedup}
\end{subfigure}\hfill
\begin{subfigure}{0.49\linewidth}
  \centering\includegraphics[width=\linewidth]{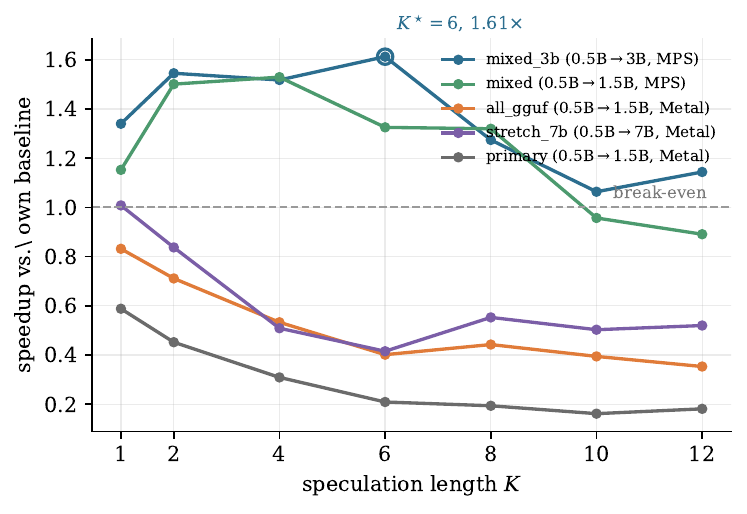}
  \caption{Speedup vs.\ $K$, each normalized by its own baseline.}
  \label{fig:vsk}
\end{subfigure}
\caption{(a) Main-run throughput ($128$ tokens, $28$ prompts); speedup factors labeled. Only
the two configurations with a Transformers-on-MPS target exceed $1.0\x$. (b) The $K$-sweep
($96$ tokens, $8$ recorded prompts): \cfg{mixed\_3b} peaks at $K{=}6$ ($1.61\x$); the three
llama.cpp-Metal-verifier configurations sit below break-even at every $K$.}
\label{fig:throughput}
\end{figure}

Table~\ref{tab:main} reports the headline numbers. \cfg{mixed\_3b} (a $0.5$B q8\_0 draft
against a $3$B fp16 MPS target) is the fastest at $1.57\x$; \cfg{mixed} follows at $1.40\x$;
the three llama.cpp-Metal-verifier configurations all lose. Table~\ref{tab:sweep} sweeps $K$
for the best configuration.

\begin{table}[t]
\centering
\caption{$K$-sweep for the best configuration (\cfg{mixed\_3b}: $0.5$B q8\_0 draft, $3$B fp16
MPS target; $96$ tokens, $8$ prompts; baseline $16.85 \pm 0.07$ tok/s).}
\label{tab:sweep}
\small
\begin{tabular}{cccc}
\toprule
$K$ & tok/s & Acceptance & Speedup \\
\midrule
$1$  & $22.57{\pm}1.09$ & $69.7\%$ & $1.34\x$ \\
$2$  & $26.03{\pm}1.60$ & $60.2\%$ & $1.55\x$ \\
$4$  & $25.57{\pm}1.93$ & $42.1\%$ & $1.52\x$ \\
$\mathbf{6}$ & $\mathbf{27.17{\pm}4.70}$ & $37.8\%$ & $\mathbf{1.61\x}$ \\
$8$  & $21.46{\pm}3.22$ & $25.6\%$ & $1.27\x$ \\
$10$ & $17.92{\pm}2.71$ & $22.2\%$ & $1.06\x$ \\
$12$ & $19.27{\pm}4.57$ & $22.8\%$ & $1.14\x$ \\
\bottomrule
\end{tabular}
\end{table}

The measured optimum is $K = 6$ at $1.61\x$, but the honest reading of Table~\ref{tab:sweep}
is a \emph{plateau}: $K \in \{2, 4, 6\}$ all deliver $1.52$--$1.61\x$ and their standard
deviations overlap ($K{=}6$'s std of $4.70$ tok/s spans the $K{=}2$ mean); beyond the
plateau, throughput decays as acceptance collapses. The companion configuration \cfg{mixed}
peaks at $K = 4$ ($43.02$ tok/s, $1.53\x$) with the same plateau shape. The three
configurations below $1.0\x$ in Figure~\ref{fig:vsk} decelerate at nearly \emph{every} $K$
(\cfg{stretch\_7b} only ties its baseline, $1.01\x$, at $K{=}1$);
Sections~\ref{sec:time}--\ref{sec:parallel} show precisely why.

\subsection{Acceptance by draft position}
\label{sec:accept}

\begin{figure}[t]
\centering
\includegraphics[width=0.66\linewidth]{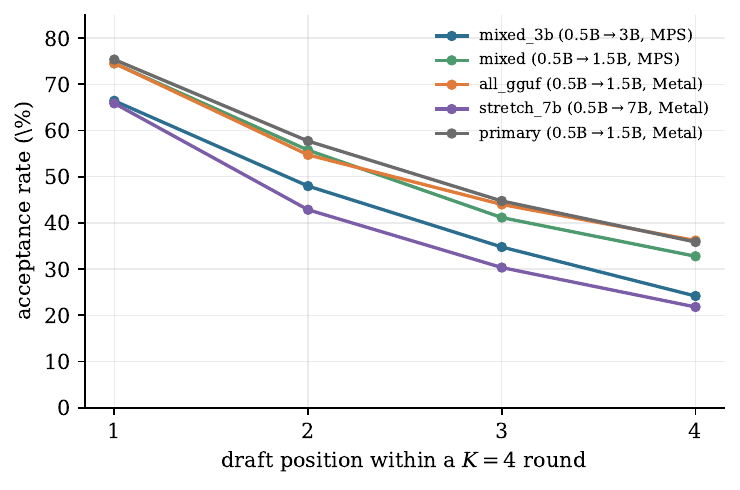}
\caption{Acceptance rate by draft position within $K{=}4$ rounds, all five main-run
configurations. Position $1$ is accepted far more often than position $4$ --- the signature
of compounding draft/target divergence. Configurations with a larger draft/target quality
gap (\cfg{stretch\_7b}, \cfg{mixed\_3b}) decay fastest.}
\label{fig:accept}
\end{figure}

For the \cfg{primary} main run, positional acceptance falls monotonically: $75.4\%$,
$57.7\%$, $44.7\%$, $35.9\%$ for positions $1$--$4$; for \cfg{mixed\_3b} (a larger
draft/target quality gap): $66.4\%$, $48.0\%$, $34.8\%$, $24.2\%$ (Figure~\ref{fig:accept}).
This shape is the mechanism behind the $K$-sweep: each additional draft position is
conditioned on a longer \emph{unverified} prefix, so its acceptance probability is roughly
the product of survival through all earlier positions times a decreasing marginal agreement.
Sweep-level pooled acceptance tells the same story from another angle, declining from
$69.7\%$ at $K{=}1$ to $22.8\%$ at $K{=}12$ for \cfg{mixed\_3b}. The practical consequence:
the marginal value of drafting token $K{+}1$ shrinks geometrically, while its cost (one more
draft step, one more verify position) is constant or growing --- producing an interior
optimum in $K$.

\subsection{Where the time goes}
\label{sec:time}

\begin{figure}[t]
\centering
\includegraphics[width=0.62\linewidth]{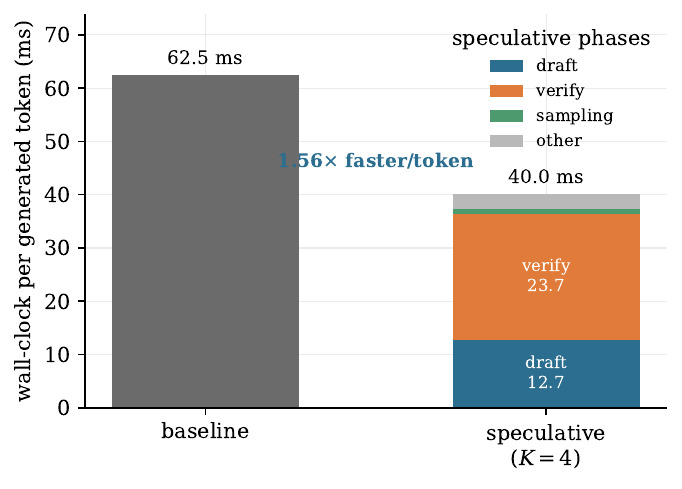}
\caption{Per-token latency decomposition for the best configuration (\cfg{mixed\_3b},
$K{=}4$) against its baseline. Sampling (the correctness-critical CPU-float32 arithmetic) is
$0.88$\,ms/token --- $2.2\%$ of wall time. The draft is cheap enough ($12.7$\,ms total per
token) that a $43\%$ acceptance rate still nets a large win.}
\label{fig:latency}
\end{figure}

For \cfg{mixed\_3b} at $K = 4$: $40.0$\,ms per generated token, split into draft $12.7$\,ms,
verify $23.7$\,ms, sampling $0.88$\,ms, other $2.8$\,ms --- against a $62.5$\,ms/token
baseline ($1000/16.0$; Figure~\ref{fig:latency}). Two facts matter. First, \emph{sampling
overhead is negligible} ($2.2\%$ of wall time), vindicating the CPU-float32 design:
correctness-critical math need not live on the accelerator. Second, the engine spends $32\%$
of its time drafting --- at $8.5$\,ms per draft step against the $3$B target's roughly
$64$\,ms verify round, the draft is cheap enough that even a $43\%$ acceptance rate at
$K{=}4$ nets a large win.

The same decomposition explains the \cfg{primary} failure: its Transformers-fp16-on-MPS
draft costs $18.0$\,ms \emph{per draft step} --- more than the llama.cpp $1.5$B target's
entire $13.6$\,ms baseline step. A draft slower than its target cannot produce speedup at any
$K$ or acceptance rate; framework per-step overhead (Python dispatch, MPS kernel launch)
dominates model size at this scale. Swapping the same $0.5$B draft to llama.cpp q8\_0
($8.4$\,ms/step, \cfg{all\_gguf}) fixes the draft but still loses --- for the reason below.

\subsection{When ``parallel'' verification is not parallel}
\label{sec:parallel}

Speculative decoding's arithmetic rests on one systems assumption: scoring $K{+}1$ known
positions costs roughly one decode step, not $K{+}1$. On the quantized Metal backend, that
assumption is measurably false for the relevant batch sizes. Verify cost per round on the
$7$B q4\_k\_m target grows nearly linearly --- $51.0$\,ms for a $2$-token pass, $74.7$\,ms
for $3$, $164.9$\,ms for $5$, $218.5$\,ms for $7$ --- then \emph{drops} to a flat
$141$--$142$\,ms at batches of $9$--$13$, the signature of a kernel-path switch: below the
threshold, llama.cpp's Metal path executes quantized matrix--\emph{vector} products per
position, re-reading the weights each time; above it, a matrix--matrix kernel reads the
weights once. By contrast, the Transformers-MPS backend extends its cache with $5$ tokens in
$42.3$\,ms versus $35.2$\,ms for one token --- a $1.20\x$ ratio, i.e., genuinely
batch-parallel.

This single backend property cleanly partitions Figure~\ref{fig:vsk}: the two configurations
whose \emph{verifier} runs on torch-MPS GEMM sit above $1.0\x$; the three whose verifier is
llama.cpp-Metal at batch sizes $2$--$7$ sit below it, regardless of how favorable their draft
economics are. \cfg{stretch\_7b} is the sharpest illustration: a textbook-asymmetric pair
($8.5$\,ms draft steps against a $43$\,ms/token target) that still \emph{halves} throughput
at $K{=}4$ because each $5$-token verify pass costs nearly four single-token steps.

\section{Discussion and Limitations}

\paragraph{Why the optimal $K$ sits where it does.}
Three forces set the optimum. (i)~The geometric decay of positional acceptance
(Section~\ref{sec:accept}) means no single constant $\alpha$ reproduces a round's yield in
the idealized $(1-\alpha^{K+1})/(1-\alpha)$: measured yield for \cfg{mixed\_3b} at $K{=}4$ is
$2917/1084 \approx 2.7$ tokens per round, above the $\approx 1.7$ that the pooled rate
$\alpha = 0.43$ gives in that formula, because the pooled average understates the high
early-position acceptance ($66.4\%$ at position~$1$) that dominates a round while later
positions ($24.2\%$ at position~$4$) add little. (ii)~Draft cost grows linearly in $K$
($8.5$\,ms per step). (iii)~Verify cost grows with batch size --- mildly on MPS, brutally on
quantized Metal. On the MPS-verifier configurations these forces balance at $K \in [2, 6]$;
on the Metal-verifier configurations force~(iii) is so strong that the optimum collapses to
$K = 1$ (or, equivalently, to not speculating at all). A rough reading of
Table~\ref{tab:sweep}: on the MPS-verifier configurations the optimum sits where pooled
acceptance has fallen to the high-$30\%$s ($37.8\%$ at $K{=}6$), and throughput collapses
once acceptance drops into the mid-$20\%$s ($K \ge 8$).

\paragraph{Where the overhead lives.}
Across all five configurations, rejection-sampling arithmetic is under $1$\,ms per generated
token --- the mathematically delicate part of speculative decoding is computationally free.
The engineering-relevant costs are the two systems terms: per-step framework overhead in the
draft loop, and the batch-scaling behavior of the verifier. Both are properties of the
\emph{runtime}, not the algorithm, and both flipped configurations between speedup and
slowdown in our measurements. We suggest that implementations report a
draft-step/target-step latency ratio and a verify-batch scaling curve alongside acceptance
rates; without them, acceptance is uninterpretable as a predictor of speedup.

\paragraph{Limitations.}
(1)~All results are single-stream, batch-size-one, on one consumer machine; server-side
batching changes the economics entirely, and our wall-clock constants are specific to
Apple-silicon Metal/MPS kernels as of the pinned software versions. (2)~Acceptance rates are
properties of the specific draft/target pair (same family, instruction-tuned, $T = 0.8$);
different families, temperatures, or domains will shift every curve. (3)~The 8B-class
pairing is provided as a Colab script but was not executed for this paper; no numbers are
claimed for it. (4)~The sampled-leg equivalence test has the statistical resolution of about
$9{,}200$ tokens per method --- it bounds, rather than eliminates, distributional deviation;
the exact result is established analytically (Section~\ref{sec:proof}) and bit-exactly in the
greedy limit. (5)~The greedy real-model leg required an fp32-CPU backend because
quantized-Metal logits are not batch-invariant; we report this as a measured backend
property, but it means bit-exact greedy equivalence on the quantized backend itself is
unattainable \emph{by any implementation} whose baseline and verifier use different batch
shapes. (6)~Peak-memory figures are process-wide high-water marks and thus upper bounds per
configuration. (7)~Sweep points use $8$ recorded prompts; adjacent-$K$ differences within one
standard deviation (e.g., the $K{=}2$--$6$ plateau) should be read as ties.

\section{Conclusion and Future Work}

We built a speculative decoding engine from scratch, showed --- statistically at three
levels, and bit-exactly in the greedy limit --- that it preserves the target model's output
distribution, and measured when it actually pays on consumer hardware. The headline is a
$1.61\x$ single-stream speedup at $K{=}6$ (plateau $1.52$--$1.61\x$ across $K{=}2$--$6$) with
a $0.5$B-draft/$3$B-target pair; the anti-headline is that three plausible configurations
decelerate, for two measured reasons that have nothing to do with the algorithm: a draft
whose per-step framework overhead exceeds a small target's entire decode step, and a
quantized backend that verifies serially below batch size $8$. The guarantee is free; the
speedup is a systems property.

Future work follows directly. \emph{Continuous batching}~\citep{yu2022orca} interacts with
speculation non-trivially --- accepted-length variance de-synchronizes streams within a
batch. \emph{Tree verification} (Medusa-style~\citep{cai2024medusa}) amortizes one verify
pass over multiple candidate continuations, raising expected yield at fixed verify cost ---
particularly attractive where verify passes are expensive. \emph{Adaptive $K$} could exploit
the strong per-round acceptance signal we observe (positional decay curves in
Section~\ref{sec:accept}) to shrink $K$ when recent acceptance drops, harvesting the plateau
without paying the tail. \emph{Self-speculation}~\citep{zhang2023draft} removes the second
model entirely, which matters on memory-constrained devices like the $18$\,GB machine used
here. Finally, the Colab script ships with the repository precisely so that the 8B-class CUDA
measurement --- where batched verification is known to be parallel-efficient --- can be added
to this paper's tables by anyone with a free T4.

\section*{Originality Note}

All prose in this paper is original writing; no sentence or distinctive phrasing is
reproduced from the cited works. All equations are typeset fresh, and all figures are
generated by this project's own code from this project's own measurements. Every empirical
value is traceable to a specific artifact under \texttt{results/} in the accompanying
repository, and the figure-generating scripts are included so each plotted value can be
traced to source data. Text-to-text generative AI (Anthropic's Claude, used via Claude Code)
assisted in drafting prose and \LaTeX{}; all bibliographic entries were verified against
their source pages, and the authorship of and responsibility for the method, code,
experiments, and claims rest with the author.

\section*{Reproducibility}

The engine, benchmark harness, equivalence gate, vendored prompt set, figure scripts, and
all result JSONs are in the accompanying repository. The frozen interface file fixes the
model API, cache choreography, and result schema; the benchmark and gate scripts regenerate
every artifact cited here, and \texttt{paper/figures/scripts/make\_figures.py} regenerates
every figure from those artifacts.

\bibliographystyle{plainnat}
\bibliography{references}

\end{document}